\documentclass[11pt]{article}

\usepackage[final]{acl}

\usepackage{times}
\usepackage{latexsym}
\usepackage{booktabs}
\usepackage{multicol}
\usepackage{multirow}
\usepackage{graphicx}
\usepackage{amsmath}
\usepackage{amssymb}
\usepackage{bm}
\usepackage{amsthm}
\usepackage{xspace}

\newtheorem{theorem}{Theorem}

\theoremstyle{definition}
\newtheorem{definition}{Definition}
\newcommand{\EHGRAG}{EHRAG\xspace}

\usepackage[utf8]{inputenc}
\usepackage{booktabs}
\usepackage{tabularx}
\usepackage{pifont}
\usepackage[table]{xcolor} 
\usepackage{enumitem}

\newcommand{\cmark}{\ding{52}}
\newcommand{\xmark}{\ding{56}}

\definecolor{headergray}{HTML}{EFEFEF}
\definecolor{baselinepink}{HTML}{FFF0F0}
\definecolor{proposelblue}{HTML}{F0F7FF}

\usepackage{microtype}
\usepackage{inconsolata}
\usepackage{graphicx}

\title{\EHGRAG: Bridging Semantic Gaps in Lightweight GraphRAG via Hybrid Hypergraph Construction and Retrieval}

\author{
  Yifan Song\textsuperscript{1}, \;
  Xingjian Tao\textsuperscript{1}, \;
  Zhicheng Yang\textsuperscript{1}, \;
  Yihong Luo\textsuperscript{2}, \;
  and Jing Tang\textsuperscript{1,2}\thanks{Corresponding Author: Jing Tang.} \\
  \textsuperscript{1}The Hong Kong University of Science and Technology (Guangzhou) \\
  \textsuperscript{2}The Hong Kong University of Science and Technology \\
  \texttt{ysong853@connect.hkust-gz.edu.cn}, \;
  \texttt{jingtang@ust.hk}
}

\begin{document}
\maketitle
\begin{abstract}

Graph-based Retrieval-Augmented Generation (GraphRAG) enhances LLMs by structuring corpus into graphs to facilitate multi-hop reasoning. While recent lightweight approaches reduce indexing costs by leveraging Named Entity Recognition (NER), they rely strictly on structural co-occurrence, failing to capture latent semantic connections between disjoint entities. To address this, we propose \textbf{\EHGRAG}, a lightweight RAG framework that constructs a hypergraph capturing both structure and semantic level relationships, employing a hybrid structural-semantic retrieval mechanism. Specifically, \EHGRAG constructs structural hyperedges based on sentence-level co-occurrence with lightweight entity extraction and semantic hyperedges by clustering entity text embeddings, ensuring the hypergraph encompasses both structural and semantic information. For retrieval, \EHGRAG performs a structure-semantic hybrid diffusion with topic-aware scoring and personalized pagerank (PPR) refinement to identify the top-k relevant documents. Experiments on four datasets show that \EHGRAG outperforms state-of-the-art baselines while maintaining linear indexing complexity and zero token consumption for construction. Code is available at \url{https://github.com/yfsong00/EHRAG}.
\end{abstract}
\begin{sloppy}
\section{Introduction}
\label{sec:intro}

Retrieval Augmented Generation (RAG) emerges as a promising paradigm for minimizing hallucinations in Large Language Models (LLMs) by based on external knowledge~\cite{lewis2020retrieval,shuster2021retrieval,ji2023survey,gao2023retrieval,asai2023self,qiu2025efficient,tao2024llms,linghu2025llm}. While standard RAG systems excel at retrieving explicit information, they often struggle with complex queries that require multi-hop reasoning across disparate documents. To address this, Graph-based RAG (GraphRAG) approaches~\cite{edge2024local} have been introduced, which organize the corpus as a structured graph to facilitate multi-step information propagation.

Despite their effectiveness, existing GraphRAG methods face a significant efficiency bottleneck~\cite{han2025rag,peng2024graph}. Traditional approaches typically rely on LLMs to extract entity-relation triplets and then construct the knowledge graphs. This process incurs prohibitive computational costs, which makes them impractical for large-scale corpus. To mitigate this, recent lightweight frameworks such as LinearRAG~\cite{zhuang2025linearrag} have proposed replacing expensive LLM-based relation extraction with efficient Named Entity Recognition (NER) and modeling document structures directly.

However, we argue that these lightweight approaches introduce a critical limitation: the \textbf{Semantic Gap}. By relying solely on explicit structural co-occurrence (i.e., entities appearing in the same sentence or document), these methods fail to capture latent connections between semantically related but structurally disjoint entities. Consider the example illustrated in Figure~\ref{fig:teaser}. To answer the question \textit{'Who is the spouse of the current monarch of the UK?'}, a system must connect 'monarch' in document B to 'Queen' in document A. Structure-only graph would treat these as distinct nodes because they are not same phrase and never appear in the same context window. Consequently, the reasoning chain is broken, which leads to retrieval failure. 

To bridge this gap without sacrificing efficiency, we propose \textbf{\EHGRAG} (\underline{E}fficient \underline{H}ypergraph-based \underline{RAG}), a novel framework that unifies structural and semantic indexing via hypergraphs. Unlike simple graphs, hypergraphs use hyperedges to connect arbitrary sets of nodes, making them naturally suitable for modeling both explicit document inclusion (structural hyperedges) and implicit semantic clusters (semantic hyperedges). Specifically, we employ lightweight NER for node extraction and construct structural hyperedges based on sentence-level co-occurrence. Simultaneously, we cluster entity embeddings to form semantic hyperedges, which connect semantically similar entities across different documents. During retrieval, we introduce a hybrid retrieval mechanism that utilizes diffusion-based entity activation to propagate query influence through explicit and latent pathways, followed by topic-aware passage scoring and PPR refinement to accurately identify the top-$k$ relevant documents based on both content and graph structure.

\begin{figure}[t]
    \centering
    \includegraphics[width=\columnwidth]{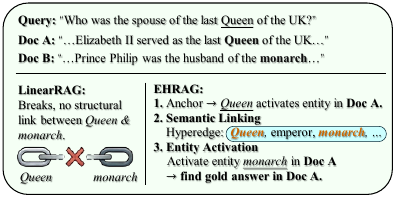} 
    \caption{An example of the semantic gap. While structural graph fails to link disjoint entities (monarch and Queen), \EHGRAG bridges them via a semantic hyperedge, enabling the retrieval of the multi-hop reasoning.}
    \label{fig:teaser}
\vspace{-3mm}
\end{figure}

To validate the effectiveness of \EHGRAG, we conduct experiments on four multi-hop QA benchmarks. Experimental results demonstrate that our method significantly outperforms state-of-the-art baselines including HippoRAG2 and LinearRAG, improving accuracy by up to 6.9\% on 2WikiMultiHop while maintaining zero token consumption for indexing.

In summary, our contributions are as follows:
\begin{itemize}
\item We identify the \textit{semantic gap} in existing lightweight GraphRAG methods, where the lack of semantic connectivity compromises multi-hop reasoning.
\item We propose \EHGRAG, a hybrid hypergraph-based RAG framework that efficiently integrates structural co-occurrence and latent semantic correlations via entity activation by hypergraph diffusion and well designed passage scoring strategy.
\item Extensive experiments across four benchmarks demonstrate that \EHGRAG achieves state-of-the-art performance with linear indexing complexity, offering a scalable solution for knowledge-intensive tasks.
\end{itemize}
 \section{Preliminaries}
\subsection{Backgrounds}

Given a corpus $\mathcal{C} = \{d_1, \dots, d_{|\mathcal{C}|}\}$ and a query $q$, retrieval-augmented generation (RAG) aims to retrieve relevant documents to support answer generation. While vanilla RAG relies on vector similarity, it often treats documents as independent units, struggling with complex queries that require multi-hop reasoning across interconnected information. To overcome this limitation, graph-based RAG structure the corpus into knowledge graphs to improving multi-step information propagation. However, traditional GraphRAG methods typically depend on LLMs to extract fine-grained entity-relation triples, which incurs prohibitive computational costs and latency. 

Recently, lightweight methods (e.g., LinearRAG) replace expensive relation extraction with Named Entity Recognition (NER), modeling relationships primarily through explicit structural co-occurrence. Despite their efficiency, these pairwise graph structures often fail to capture latent semantic correlations between spatially disjoint entities. To bridge this gap, we model the corpus as a hypergraph $\mathcal{H} = (\mathcal{V}, \mathcal{E})$, where the node set $\mathcal{V}$ consists of unique entities extracted via lightweight NER. Unlike simple edges, a hyperedge $e \in \mathcal{E}$ connects an arbitrary subset of entities ($e \subseteq \mathcal{V}$), enabling the unified modeling of both explicit structural inclusion (entities within a sentence) and implicit semantic grouping (entities within a semantic cluster). The hypergraph is always represented by an incidence matrix $\mathbf{H} \in \{0, 1\}^{|\mathcal{V}| \times |\mathcal{E}|}$, where $H(v, e) = 1$ indicates entity $v$ belongs to hyperedge $e$.

\subsection{Lightweight Node Extraction}

Efficient graph-based RAG uses lightweight Named Entity Recognition (NER) models (e.g., GLiNER or SpaCy) to extract entities directly from the raw text, thereby bypassing the computationally intensive relation extraction (RE) process. For each sentence $s_{i,j}$ in document $d_i$, we extract a set of entities $E_{i,j}$. The global node set $\mathcal{V}$ is the union of all unique entities: $\mathcal{V} = \bigcup_{d_i \in \mathcal{C}} \bigcup_{s_{i,j} \in d_i} E_{i,j}$. This process scales linearly with the corpus size, effectively avoiding the $\mathcal{O}(N^2)$ complexity associated with pairwise relation modeling.

\section{Methodology}
\label{sec:method}
To overcome the limitations of existing lightweight GraphRAG methods, we propose \textbf{\EHGRAG}, an efficient hpyergraph-based retrieval framework. The whole process of \EHGRAG is shown in Figure \ref{fig:framework}. In short, \EHGRAG models the corpus as a hybrid hypergraph $\mathcal{H} = (\mathcal{V}, \mathcal{E})$, where $\mathcal{V}$ denotes the set of nodes (i.e. entities in the corpus), $\mathcal{E} = \mathcal{E}_{str} \cup \mathcal{E}_{sem}$ is the set of hyperedges. Here $\mathcal{E}_{str}$ represents the hyperedges generated from the structure of the documents and $\mathcal{E}_{sem}$ represents the hyperedges capturing implicit semantic relations. For each question, \EHGRAG first finds the similar entities as anchors in the hypergraph and then activates the relevant entities to construct a subgraph. Finally, it scores the passages with different dimensions and uses PPR to refine the scores and get the final top-$k$ results. Next, we will give the details of each component of \EHGRAG and analyze why it works. Due to the page limit, we put the theoretical analysis of \EHGRAG in Appendix \ref{sec:theory}.

\begin{figure*}[t]
    \centering
    \includegraphics[width=\textwidth]{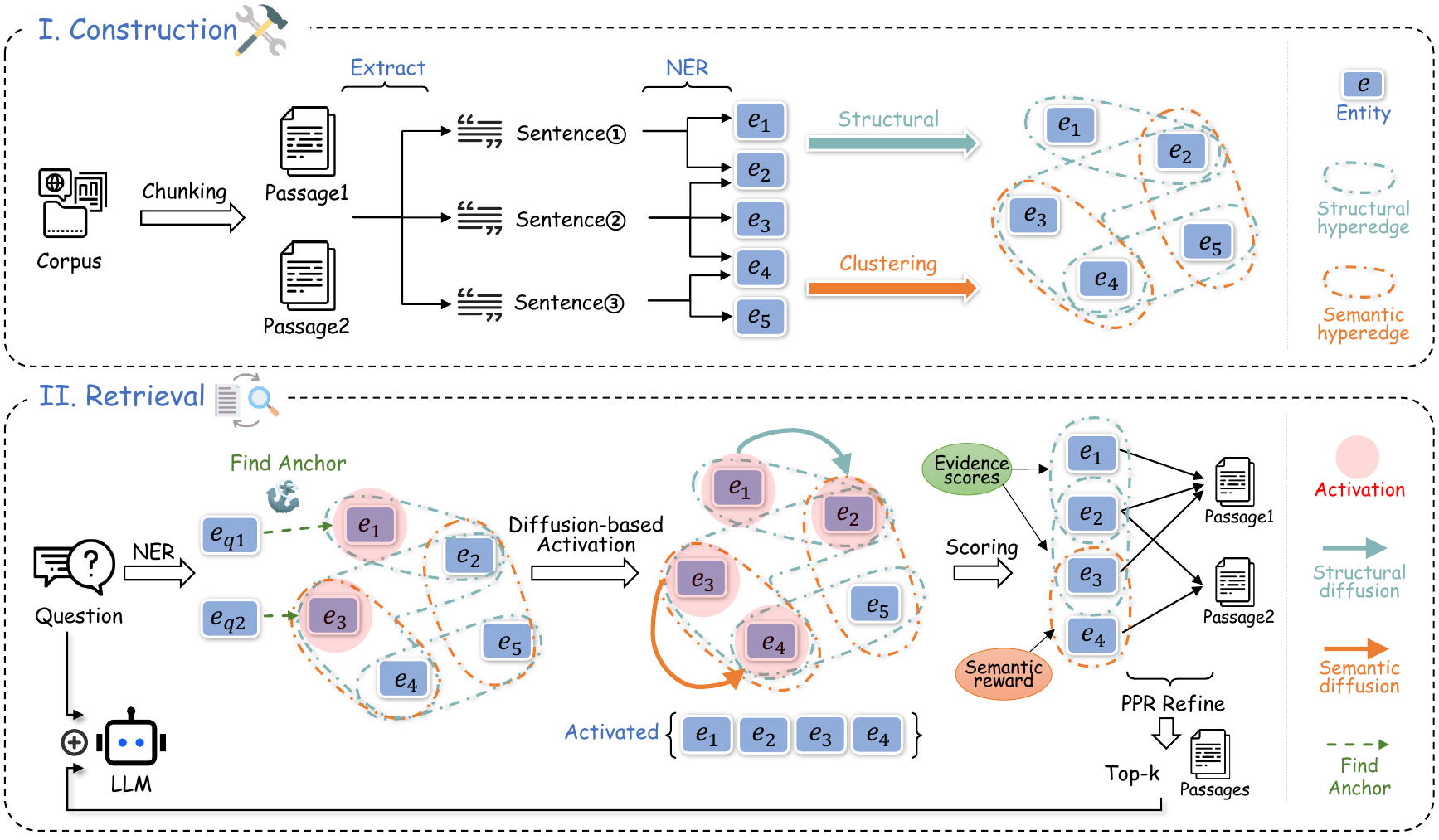} 
    \caption{The overall framework of \textbf{\EHGRAG}. The process is divided into two phases: (1) \textbf{Offline Construction}: We extract entities using lightweight models to build structural hyperedges, while simultaneously clustering entity embeddings (e.g., BIRCH) to form semantic hyperedges. (2) \textbf{Online Retrieval}: User queries activate anchor nodes, initiating a structure-semantic hybrid diffusion process that propagates scores through the hypergraph to rank and retrieve relevant documents for generation.}
    \label{fig:framework}
\end{figure*}

\subsection{Hybrid Hypergraph Construction}
\label{sec:construction}
Unlike general graphs where edges connect pairs of nodes, a hyperedge $e \in \mathcal{E}$ in a hypergraph can connect an arbitrary number of nodes, making it ideal for modeling complex group relationships in RAG. We construct two categories of hyperedges: \textit{Structural Hyperedges} for local context preservation and \textit{Semantic Hyperedges} for global concept linking.

\paragraph{Structural Hyperedges: Explicit Context Modeling.}
To capture the precise context, we treat each sentence as a structural unit. Let $\mathcal{S} = \{s_1, s_2, \dots, s_N\}$ be the set of sentences in the corpus. Following previous lightweight framework \cite{zhuang2025linearrag,zhao20252graphrag}, we employ the lightweight Named Entity Recognition (NER) module (e.g. spaCy package) to extract the entity set $\mathcal{V}$. Evidently, entities appearing in the same sentence exhibit contextual relevance. So we construct a \textbf{structural hyperedge} $e_{s_j}^{str} = \{v_i \in \mathcal{V} \mid v_i \in s_j\}$ for each sentence $s_j$, which contains all entities explicitly appearing within it. Then we can get the incidence matrix $\mathbf{H}^{str} \in \mathbb{R}^{|\mathcal{V}| \times |\mathcal{S}|}$ of the structural hyperedges, which is defined as $\mathbf{H}^{str}_{i, j} = 1$ if $v_i \in e_{s_j}^{str}$ and $0$ otherwise. This construction ensures that entities within the same context window form a fully connected clique in the primal projection, preserving the integrity of local information.

\paragraph{Semantic Hyperedges: Latent Topic Discovery.}
Since only same entities will be combined cross the documents, graphs with only structural connections will fail to link synonymous entities appearing in disjoint contexts (e.g., GPU and Graphics Card). To bridge this semantic gap, we introduce \textbf{semantic hyperedges} via density adaptive projecting in the latent space. Let $\mathbf{x}_i \in \mathbb{R}^d$ be the text embedding of entity $v_i$, entities belonging to the same latent topic cluster around a centroid in the embedding space. Instead of rigid partitioning, we employ a clustering algorithm (e.g., BIRCH) to dynamically identify $K$ latent topic centroids $\mathcal{C} = \{\mathbf{c}_1, \dots, \mathbf{c}_K\}$ without pre-specifying $K$.

For each centroid $\mathbf{c}_k$, we define a semantic hyperedge $e_{sem}^k$ that bonds the top-$D$ semantically closest entities. To incorporate the semantic uncertainty, we assign a continuous weight to the incidence entry based on the kernel distance:
\begin{equation}
    \mathbf{H}^{sem}_{i, k} = 
    \begin{cases} 
        \scriptstyle \exp\left(-\frac{\|\mathbf{x}_i - \mathbf{c}_k\|^2}{\tau}\right) & v_i \in \mathcal{N}_D(\mathbf{c}_k), \\
        \scriptstyle 0 & \text{otherwise},
    \end{cases}
\end{equation}
where $\mathcal{N}_D(\mathbf{c}_k)$ denotes the set of $D$ nearest neighbors to $\mathbf{c}_k$, and $\tau$ is a temperature parameter. This topology allows information to teleport between structurally distant but semantically related entities.

\subsection{Structure-Semantic Mixed Retrieval}
\label{sec:retrieval}
Based on the constructed hybrid hypergraph, \EHGRAG includes a novel two-stage retrieval mechanism: \textit{Diffusion-based Activation} followed by \textit{Topic-aware Passage Scoreing}. To sum up, it first activates the nodes in the hypergraph to generate a subset of entities that are relevant with the question. Then \EHGRAG scores the passages based on the connection between passages and the subset of entities with different views and uses personalized pagerank (\textbf{PPR}) to refine the scores to generates the final top-$k$ passages for retrieval. We will start from the initialization stage to present the details of the retrieval mechanism.

\subsubsection{Anchors Initialization}
To initiate the retrieval, we first map the user query $q$ to the graph. Following existing methods~\cite{zhuang2025linearrag,edge2024local,zhao20252graphrag}, we extract a set of query entities $E_q$ using the same NER module during hypergraph construction. For each query entity $e \in E_q$, we compute its embedding similarity with all node entities in $\mathcal{V}$ and select the most similar nodes as the anchor to start the retrieval.
We define the initial activation state $\mathbf{a}^{(0)}$ for each anchor as $\mathbf{a}^{(0)}(v) = \cos(\mathbf{x}_e, \mathbf{x}_v)$, which ensures that the retrieval starts from multiple robust entry points, tolerating extraction errors or minor morphological variations.

\subsubsection{Diffusion-based Entities Activation}
Starting from the seed activation $\mathbf{a}^{(0)}$, we propagate scores to identify contextually linked entities. Based on our implementation, this process is designed as a **Two-Phase Diffusion**: a one-off Semantic Expansion followed by an Iterative Structural Propagation.

\paragraph{Semantic Diffusion for Topic Projection.}
Before the iterative search, we first expand the seed set to include latent synonymous entities. This corresponds to a weighted projection through the semantic incidence matrix $\mathbf{H}^{sem}$.
For the seed activation vector $\mathbf{a}^{(0)}$, we calculate the semantic expansion vector $\mathbf{a}_{sem}$ as:
\begin{equation}
    \mathbf{a}_{sem} = \gamma \cdot \mathbf{H}^{sem} (\mathbf{H}^{sem})^\top \mathbf{a}^{(0)},
\end{equation}
where $\gamma$ is a decay factor that is used to avoid the semantic shift during multi-hop diffusion. The operator $\mathbf{H}^{sem} (\mathbf{H}^{sem})^\top$ represents hopping from entities to their cluster centroids and then travel to the similar entities.
Then we can get the initial search frontier with potential topic as $\mathbf{a}^{(1)} = \mathbf{a}^{(0)} + \mathbf{a}_{sem}$ while the initial global weight vector is $\mathbf{w} = \mathbf{a}^{(1)}$.

\paragraph{Iterative Structural Diffusion.}
We then perform a $T$-step structural diffusion to capture local context. Unlike standard random walks, edge weights in our graph are dynamically modulated by the query. Let $\mathbf{a}^{(t)}$ be the active frontier at iteration $t$. The propagation to the next hop consists of three steps:

\textbf{Step 1: Entity-to-Sentence Projection.} First, activation flows from entities to the sentences containing them:
\begin{equation}
    \mathbf{s}^{(t)} = (\mathbf{H}^{str})^\top \mathbf{a}^{(t)},
\end{equation}
where $\mathbf{s}^{(t)} \in \mathbb{R}^{|\mathcal{S}|}$ represents the activation potential of each sentence.

\textbf{Step 2: Query-Gated Filtering.} Not all sentences containing active entities are relevant. We strictly gate the passage flow by calculating the similarity between the sentence embeddings $\mathbf{E}_S$ and the query embedding $\mathbf{x}_q$. We construct a dynamic diagonal gating matrix $\mathbf{G}_q$:
\begin{equation}
    \mathbf{G}_q[j,j] = 
    \begin{cases} 
        \mathbf{e}_{s_j}^\top \mathbf{x}_q & \text{if } s_j \in \mathbb{F}(\mathbf{E}_S \mathbf{x}_q, L), \\
        0 & \text{otherwise},
    \end{cases}
\end{equation}
where $\mathbb{F}(\mathbf{x},L)$ is the function that selects the top-$L$ elements from $\mathbf{x}$. This step ensures that activation only flows through sentences that are semantically relevant to the user's question.

\textbf{Step 3: Accumulative Update.} The activation flows back to entities through the gated sentences. The new activation frontier $\Delta \mathbf{a}^{(t+1)}$ is calculated as:
\begin{equation}
    \Delta \mathbf{a}^{(t+1)} = \mathbf{H}^{str} \mathbf{G}_q \mathbf{s}^{(t)}.
\end{equation}
The global weight vector $\mathbf{w}$ and the frontier for the next iteration are updated as:
\begin{equation}
    \mathbf{a}^{(t+1)} = \delta(\Delta \mathbf{a}^{(t+1)}, \epsilon),
\end{equation}
\begin{equation}
    \mathbf{w} \leftarrow \mathbf{w} + \mathbf{a}^{(t+1)},
\end{equation}
where $\delta(\mathbf{x}, \epsilon)$ is a function that only reserves the elements that are larger than $\epsilon$ in $\mathbf{x}$ and $\epsilon$ is a pruning threshold. This process repeats until convergence or maximum iterations are reached, resulting in a final weight vector $\mathbf{w}^*$ that encodes both explicit semantic similarity and implicit multi-hop contextual relevance. Then non-zero elements in $\mathbf{w}^*$ is the set of activated nodes used for scoring the related passages.

\subsubsection{Topic-aware Passages Scoring}
After $T$ iterations of diffusion, we obtain a set of activated entities and topics. To retrieve the final documents, we propose a topic-aware hybrid scoring function that evaluates document relevance from three dimensions. For the question $q$, the score of $d$-th passage $p_d$ is defined as follows:

\begin{equation}
\begin{split}
    S(d) &= \underbrace{S_{d}(q, d)}_{\text{Global Context}} + \lambda_1 \underbrace{\sum_{v \in p_d} \log(1 + \mathbf{w}(v))}_{\text{Explicit Evidence}} \\
    &\quad + \lambda_2 \cdot\underbrace{ \log(1 + \sum_{v \in \mathcal{C}_d}S_{topic}(v))}_{\text{Semantic Reward}}.
\end{split}
\end{equation}

\begin{itemize}
    \item \textbf{Global Context $S_{d}(q, d)$:} Following existing work~\cite{zhuang2025linearrag,zhao20252graphrag}, this term is derived from a standard dense retriever (i.e. dot product of the question embeddings and the passage embeddings), ensuring that the documents broadly match the intent of the query.
    \item \textbf{Explicit Evidence:} It aggregates the scores of activated entities contained in the $d$-th passage. The logarithmic term prevents documents with simple keyword repetition from dominating the ranking.
    \item \textbf{Semantic Reward:} This term captures the latent thematic relevance of the document. Here, $\mathcal{C}_d$ denotes the set of unique semantic clusters (topics) appearing in document $d$, $S_{topic}(v)$ represents the global importance of cluster $v$, calculated by aggregating the semantic activation scores $\mathbf{a}_{sem}$ of all retrieved entities belonging to this cluster. This allows the system to retrieve documents that discuss the correct concept even if they lack some exact keyword matches with the query.
\end{itemize}

\paragraph{Scores Refinement via PPR.}
To enforce global consistency, we use personalized pagerank (PPR) to refine the scores on the graph $\mathcal{G}^*=(\mathcal{V}^*, \mathcal{E}^*)$ where nodes $\mathcal{V}^* = \mathcal{V}_{ent} \cup \mathcal{V}_{psg}$ include both entities and passages. The edge set $\mathcal{E}^*$ integrates structural containment (linking entity $v_i$ to passage $p_j$ if $v_i \in p_j$) and semantic similarity (linking entities in the same cluster). The process is defined as follows:
\begin{equation}
    \mathbf{r}^{(t+1)} = (1-\alpha) \mathbf{M} \mathbf{r}^{(t)} + \alpha \mathbf{r}_{init},
\end{equation}
where $\mathbf{M}$ is the normalized adjacency matrix with both entity-passage bipartite connections and entity-entity semantic links. The restart vector $\mathbf{r}_{init}$ is initialized with $S_{init}$ scores mapped to passage nodes and the stationary distribution $\mathbf{r}^{*}$ captures the final relevance of each document that can be used to determine the final top-$k$ passages for retrieval.

\subsection{Complexity Analysis}
\label{sec:complexity}
\paragraph{Construction Complexity.} Let $L$ be the total number of tokens in the corpus, $|\mathcal{V}|$ be the number of unique entities, and $d$ be the embedding dimension. The construction phase consists of NER and Clustering. First, lightweight NER models (e.g., SpaCy) process text linearly with corpus length, so its time complexity is $\mathcal{O}(L)$. Next, the BIRCH algorithm constructs the CF-Tree in a single scan of the entity embeddings, whose complexity is $\mathcal{O}(|\mathcal{V}| \log (|\mathcal{V}|/B))$, where $B$ is the branching factor. Finally, the total construction time complexity is $\mathcal{O}(L + |\mathcal{V}|d)$. Since $|\mathcal{V}| \ll L$, the whole construction takes linear time with corpus length.

\paragraph{Retrieval Complexity.}
The retrieval overhead is dominated by sparse matrix operations on $\mathcal{G}^*$, scaling linearly with the number of edges $\text{nnz}(\mathbf{A}^*)$. Physically, these edges correspond to entity mentions within the corpus. Since the number of mentions is strictly bounded by the total token count $L$, the complexity is $\mathcal{O}(L)$. This linear scalability ensures \EHGRAG remains efficient even for long-context processing.

\section{Experiments}
\label{sec:experiments}


\subsection{Experimental Setup}

\paragraph{Datasets and Baselines.}
Following prior work~\cite{jimenez2024hipporag, zhuang2025linearrag}, we evaluate our method on four datasets including three multi-hop reasoning benchmarks: \textbf{HotpotQA} \cite{yang2018hotpotqa}, \textbf{2WikiMultiHop} \cite{xanh2020_2wikimultihop}, \textbf{MuSiQue} \cite{trivedi2022musique} and one domain-specific dataset \textbf{Medical} from GraphRAG-Bench~\cite{xiang2025use}. Detailed dataset statistics are provided in Appendix~\ref{app:datasets}. We compare \EHGRAG against two groups of baselines: (1) direct zero-shot LLM inference, including LLaMA3-8B, LLaMA3-13B~\cite{dubey2024llama3}, Qwen3-8B~\cite{yang2025qwen3}, GPT-3.5-trubo, GPT-4o-mini~\cite{openai2023gpt4} and (2) retrieval-augmented-generation methods, including vanilla RAG, KGP~\cite{wang2024knowledge}, G-retriever~\cite{he2024g}, GraphRAG~\cite{edge2024local}, RAPTOR ~\cite{sarthi2024raptor}, E$^2$GraphRAG~\cite{zhao20252graphrag}, LightRAG~\cite{guo2024lightrag}, HippoRAG~\cite{jimenez2024hipporag}, GFM-RAG~\cite{luo2025gfm}, HippoRAG2~\cite{gutiérrez2025hipporag2} and LinearRAG~\cite{zhuang2025linearrag}.

\paragraph{Evaluation Metric.}
To validate the effectiveness, we adapt two widely used metrics following existing work \cite{zhuang2025linearrag,wang2025archrag}: 1. \textbf{SubEM} utilizes whether the ground truth answer is included in the response to determine the correctness for each question. 2. \textbf{LLM-Acc} uses the LLM to assess the correctness of each response. For the Medical dataset, we only evaluate the LLM-Acc metric because the ground truth answers contain multiple statements, which makes SubEM cannot validate the effectiveness.

\paragraph{Implementation Details.}
For a fair comparison, we utilize \textbf{all-mpnet-base-v2} \cite{song2020mpnet} as the embedding model and GPT-4o-mini~\cite{openai2023gpt4} as the generator and evaluator for all methods. For the parameter $k$ in the top-$k$ retrieval, we set $k=5$ for all methods. Further details regarding machine configuration and hyperparameters are listed in Appendix~\ref{app:config}.

\subsection{Generation Performance}

\begin{table*}[ht]
\centering
\renewcommand{\arraystretch}{0.9}
\begin{tabular}{c c c c c c c c}
\toprule
\multirow{2}{*}{\textbf{Method}} & \multicolumn{2}{c}{\textbf{HotpotQA}} & \multicolumn{2}{c}{\textbf{2WikiMultiHop}} & \multicolumn{2}{c}{\textbf{MuSiQue}} & \multicolumn{1}{c}{\textbf{Medical}} \\
\cmidrule(lr){2-3} \cmidrule(lr){4-5} \cmidrule(lr){6-7} \cmidrule(lr){8-8}
& SubEM & LLM-Acc & SubEM & LLM-Acc & SubEM & LLM-Acc & LLM-Acc \\
\midrule
\multicolumn{8}{c}{\textbf{Direct Zero-shot LLM Inference}} \\
\midrule
llama-8B & 31.10 & 27.30 & 33.60 & 16.20 & 7.40 & 8.10 & 27.31 \\
llama-13B & 24.20 & 16.80 & 21.90 & 10.50 & 3.30 & 4.40 & 28.86 \\
Qwen3-8B & 25.10 & 34.70 & 29.80 & 27.10 & 6.50 & 19.80 & 58.39\\
GPT-3.5-turbo & 33.40 & 43.20 & 28.70 & 31.00 & 10.30 & 21.90 & 45.60 \\
GPT-4o-mini & 38.90 & 40.20 & 36.30 & 31.40 & 13.60 & 15.80 & 42.10 \\
\midrule
\multicolumn{8}{c}{\textbf{Retrieval-Augmented-Generation Methods}} \\
\midrule
Vanilla RAG & 55.70 & 58.60 & 48.60 & 43.00 & 26.10 & 29.60 & 61.68 \\
KGP & 61.50 & 60.90 & 31.60 & 30.00 & 25.60 & 30.10 & 54.22 \\
G-retriever & 42.20 & 40.60 & 46.60 & 27.10 & 14.40 & 15.50 & 50.36 \\
GraphRAG & 58.60 & 59.80 & 49.40 & 41.60 & 24.30 & 28.70 & 48.50 \\
RAPTOR & 55.90 & 58.30 & 50.10 & 42.10 & 23.30 & 27.40 & 55.75 \\
E$^2$GraphRAG & 61.00 & 63.90 & 54.30 & 38.10 & 23.80 & 26.20 & 58.00 \\
LightRAG & 60.30 & 59.50 & 55.20 & 39.00 & 27.40 & 28.60 & 54.36 \\
HippoRAG & 57.00 & 59.30 & 66.10 & 59.90 & 29.30 & 24.10 & 55.04 \\
GFM-RAG & 62.70 & 65.60 & 66.80 & 59.60 & 29.90 & 34.60 & 56.07 \\
HippoRAG2 & 62.90 & 64.30 & 62.70 & 55.00 & 31.00 & 35.00 & 60.77 \\
LinearRAG & \underline{64.30} & \underline{66.50} & \underline{70.20} & \underline{63.70} & \underline{33.90} & \underline{37.00} & \underline{63.72} \\
\midrule
\textbf{\EHGRAG} & \textbf{65.70} & \textbf{69.30} & \textbf{73.40} & \textbf{70.60} & \textbf{34.30} & \textbf{38.40} & \textbf{65.32} \\
\bottomrule
\end{tabular}
 \vspace{-1mm}
\caption{Result (\%) of baselines and \EHGRAG on four benchmark datasets in terms of SubEM metric and LLM Evaluation Accuracy. The best result for each dataset is highlighted in \textbf{bold}, while the second result is indicated with an \underline{underline}.}
\vspace{-2mm}
\label{tab:results}
\end{table*}

Table~\ref{tab:results} presents the results across four datasets. We observe that \EHGRAG consistently outperforms all baselines. Graph-based RAG methods generally surpass zero-shot baselines and vanilla RAG, confirming the necessity of graph construction for multi-hop reasoning. Among graph-based RAG approaches, \EHGRAG achieves superior performance. Notably, on the \textbf{2WikiMultiHop} dataset, our method achieves a substantial gain of $3.2\%$ in SubEM and $6.9\%$ in LLM-Acc over LinearRAG, which is the state-of-the-art graph-based RAG method. This dataset frequently involves reasoning chains requiring entity aliasing (e.g., linking monarch of UK to the Queen). While LinearRAG relies on structural co-occurrence within sentences, \EHGRAG utilizes semantic hyperedges and diffusion to bridge disjoint but semantically related entities, thereby enabling more robust multi-hop retrieval. Similar gains are observed on HotpotQA ($+1.4\%$ SubEM and $+2.8\%$ LLM-Acc), indicating that incorporating latent semantic correlations improves robustness without introducing significant noise.

Furthermore, on the \textbf{Medical} dataset, we observe that several LLM-heavy graph-based methods are inferior to Vanilla RAG. This phenomenon suggests that in specialized and long-context domains, relying on low-parameter LLMs for explicit graph extraction may introduce structural noise or extraction errors that degrade retrieval quality. Conversely, \EHGRAG outperforms the strongest baseline by $1.6\%$, validating the effectiveness of our proposed semantic hypergraph construction. By leveraging implicit semantic distributions rather than relying solely on rigid, LLM-extracted edges, our method maintains robustness even facing the complex domain-specific corpus.

\subsection{Ablation Study}
To understand the impact of different hyperedge types, we conduct an ablation study by removing specific components. The results are summarized in Table~\ref{tab:ablation}.

\begin{table}[ht]
\centering
\resizebox{\columnwidth}{!}{%
\begin{tabular}{l c c}
\toprule
\textbf{Variant} & \textbf{2WikiMultiHop} & \textbf{HotpotQA}\\
\midrule
\textbf{\EHGRAG (Full)} & \textbf{70.60} & \textbf{69.30} \\
w/o \textit{Sem}-Diffusion & 67.30($\downarrow 3.3\%$) & 68.50($\downarrow 0.8\%$) \\
w/o Filtering & 68.20($\downarrow 2.4\%$) & 67.90($\downarrow 1.4\%$) \\
w/o \textit{Str}-Diffusion & 66.90($\downarrow 3.7\%$) & 66.80($\downarrow 2.5\%$) \\
w/o PPR Refine & 63.10($\downarrow 7.5\%$) & 68.50($\downarrow 0.8\%$) \\
\bottomrule
\end{tabular}
}
\vspace{-1mm}
\caption{Ablation study results (LLM-Acc). \textit{Sem} means semantic and \textit{Str} means structural.}
\label{tab:ablation}
\vspace{-3mm}
\end{table}

As shown in Table~\ref{tab:ablation}, the full \EHGRAG model consistently outperforms all variants on both datasets, though the impact of specific components varies. \textit{Str}-Diffusion proves universally critical, with its removal causing significant drops on both 2WikiMultiHop (3.7\%) and HotpotQA (2.5\%), confirming the necessity of iterative structural propagation. Interestingly, \textit{Sem}-Diffusion and PPR Refine show more pronounced effects on 2WikiMultiHop (drops of 3.3\% and 7.5\%) compared to HotpotQA (0.8\% and 0.8\%). This suggests that 2WikiMultiHop contains more semantically disjoint entities and complex global dependencies, thereby relying more heavily on latent semantic bridging and global graph consistency for robust reasoning. Finally, the w/o Filtering results demonstrate that query-gated filtering consistently benefits both datasets by reducing noise. Overall, these results prove that combining structural and semantic diffusion with noise filtering and PPR refinement is essential for robust multi-hop retrieval.

\subsection{Efficiency Analysis}

We analyze the computational efficiency of \EHGRAG compared to representative baselines on the 2WikiMultiHop and HotpotQA datasets. As shown in Figure~\ref{fig:efficiency}, we report indexing time, token consumption and overall retrieval time. \EHGRAG demonstrates superior efficiency compared to LLM-heavy baselines like GraphRAG and LightRAG. It maintains zero token consumption and completes indexing in just 267.5 seconds, which is comparable to the state-of-the-art lightweight method LinearRAG, confirming that the semantic hyperedges via BIRCH clustering adds negligible computational overhead while maintaining linear complexity.

\begin{figure*}[ht] 
    \centering

    \includegraphics[width=0.9\linewidth]{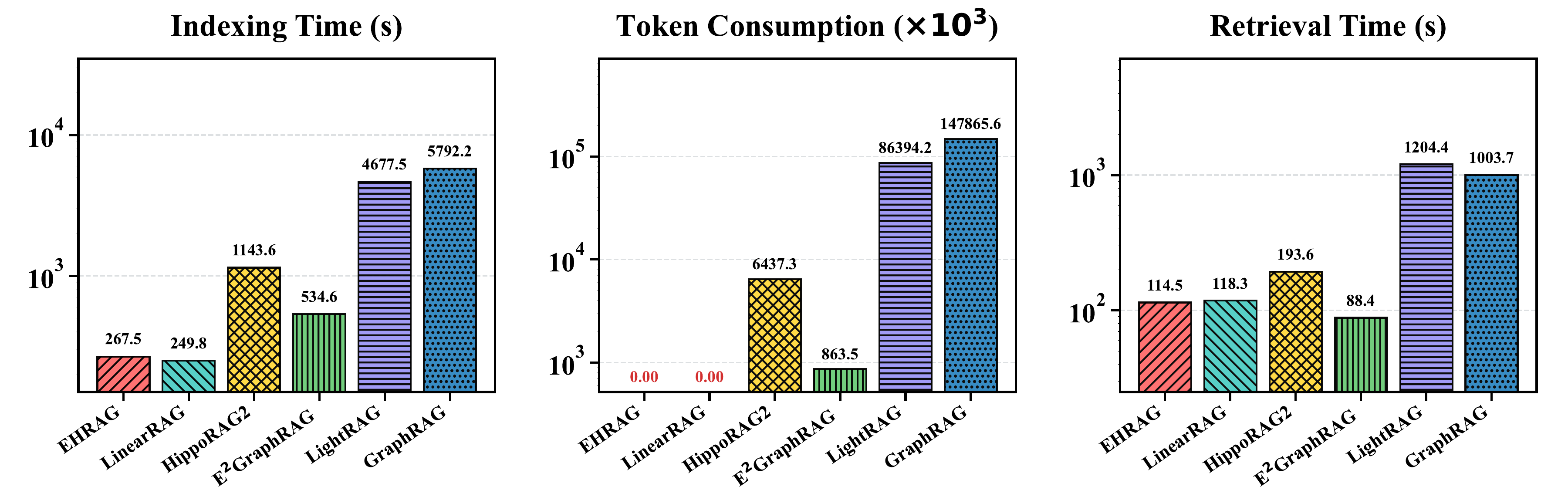}
    \vspace{-1mm} 
    \caption{Efficiency comparison on 2WikiMultiHop. We report the Indexing Time, Token Consumption, and Retrieval Time for different methods. Note that the y-axis is in log scale.}
    \label{fig:efficiency}
    \vspace{-1mm} 
\end{figure*}

In terms of retrieval efficiency, \EHGRAG outperforms most baselines including HippoRAG2 and LinearRAG. While E$^2$GraphRAG exhibits slightly lower latency, it significantly lags behind state-of-the-art graph-based RAG in generation performance and needs higher indexing overhead. Consequently, \EHGRAG achieves the best balance. It possesses the most advanced multi-hop reasoning capability, and its efficiency is comparable to the most lightweight baseline.

Because the retrieval overhead is dominated by sparse matrix operations that scale linearly, the actual graph traversal takes only a fraction of a second per query. We profiled the average inference latency per query (in milliseconds) on the 2WikiMultiHop dataset using our standard hardware setup (NVIDIA RTX 4090). The detailed breakdown is presented in Table~\ref{tab:latency_breakdown}.

\begin{table}[h]
    \centering
    \begin{tabular}{lrr}
        \toprule
        \textbf{Retrieval Stage} & \textbf{Latency} & \textbf{Percent} \\
        \midrule
        1. Lightweight NER          & 21.4 ms & 18.3\% \\
        2. Entity Embedding         & 28.8 ms & 24.7\% \\
        3. Anchor Initialization    & 5.2 ms  & 4.5\%  \\
        4. Hybrid Diffusion         & 4.6 ms  & 3.9\%  \\
        5. Evidence Scoring & 21.4 ms & 18.3\% \\
        6. Topic Scoring  & 2.1 ms  & 1.8\%  \\
        7. PPR Refinement           & 33.3 ms & 28.5\% \\
        \bottomrule
    \end{tabular}
    \caption{Per-query inference latency breakdown on 2WikiMultiHop (NVIDIA RTX 4090).}
    \label{tab:latency_breakdown}
\end{table}

This detailed breakdown clearly demonstrates that our core algorithmic contributions are highly efficient. The hybrid diffusion algorithm consumes merely 4.6 milliseconds per query (only 3.9\% of the total time). Furthermore, while the overall passage scoring step takes time, our newly introduced topic-based scoring only takes 2.1 milliseconds (1.8\%). This proves that adopting a hypergraph, performing iterative diffusion, and utilizing topic scoring does not add a heavy burden to the online inference process.

Instead, the majority of the inference time is occupied by standard pipeline components. Specifically, PPR refinement (33.3 ms), entity embedding generation (28.8 ms), standard evidence scoring (21.4 ms), and lightweight NER extraction (21.4 ms) take up the bulk of the time. Since these standard steps currently dominate the retrieval process, they represent the primary ceiling for latency. Overall, this analysis confirms that the efficiency bottleneck lies in conventional pipeline components rather than in our proposed components, demonstrating that our method achieves improved retrieval quality without sacrificing inference efficiency.

\subsection{Parameter Sensitivity Analysis}
\label{sec:param_sensitivity}

\begin{figure*}[ht]
    \centering
    \includegraphics[width=0.95\linewidth]{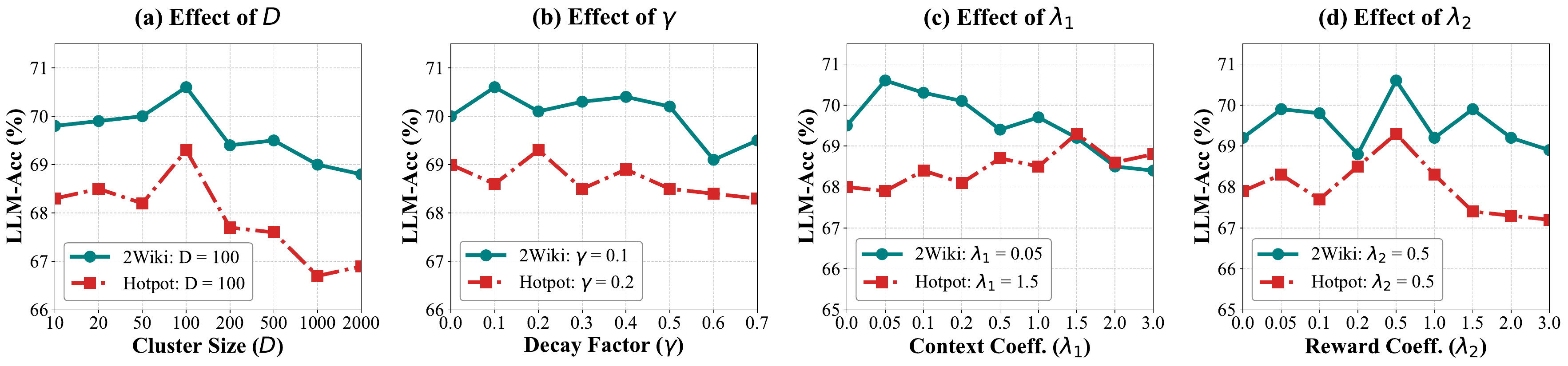}
    \vspace{-1mm}
    \caption{
    Parameter sensitivity analysis on 2WikiMultiHop (2Wiki) and HotpotQA (Hotpot) datasets. 
    }
    \label{fig:sensitivity}
    \vspace{-2mm}
\end{figure*}

To evaluate the robustness of \EHGRAG, we investigate the impact of four key hyperparameters on the 2WikiMultiHop and HotpotQA datasets including the number of nodes in a cluster $D$, the semantic propagation decay factor $\gamma$, the global context coefficient $\lambda_1$, and the semantic reward $\lambda_2$.

\textbf{Impact of Cluster Size ($D$):} The parameter $D$ controls the amount of entity within one cluster. As shown in Figure~\ref{fig:sensitivity}(a), performance initially improves as $D$ increases, peaking at $D=100$ for both datasets (70.6\% on 2WikiMultiHop and 69.3\% on HotpotQA). Setting $D$ larger than $100$ leads to a performance decline because incorporating excessive entities introduces irrelevant noise that distracts the LLM reasoning process. Thus, we recommend set $D=100$ initially for all datasets.

\textbf{Impact of Decay Factor ($\gamma$):} Figure~\ref{fig:sensitivity}(b) illustrates the effect of the semantic propagation decay factor. The model achieves optimal performance at $\gamma=0.1$ for 2WikiMultiHop and $\gamma=0.2$ for HotpotQA. This indicates that a moderate decay is necessary to maintain the focus of semantic expansion. For applying \EHGRAG on other datasets, tuning the gamma smaller than $0.5$ is recommended.

\textbf{Impact of Coefficients ($\lambda_1$ and $\lambda_2$):} Figures~\ref{fig:sensitivity}(c) and (d) analyze the balancing coefficients. For the global context coefficient $\lambda_1$, the optimal value varies significantly between datasets ($0.05$ for 2WikiMultiHop while $1.5$ for HotpotQA), suggesting that different datasets require varying degrees of global context integration. In contrast, the semantic reward coefficient $\lambda_2$ demonstrates strong stability, with both datasets achieving their peak performance at $\lambda_2=0.5$.

For other hyperparameters such as threshold $\epsilon$, sentence number $L$ and iteration number $T$, we set them based on the analysis and common values in existing studies \cite{zhuang2025linearrag,zhao20252graphrag,jimenez2024hipporag,gutiérrez2025hipporag2}. The details of all hyperparameters are listed in Appendix \ref{app:config}.

\section{Conclusion}
\label{sec:conclusion}

In this paper, we introduced \textbf{\EHGRAG}, a novel lightweight graph-based RAG framework that addresses the semantic limitations of existing efficient graph-based retrieval methods. Unlike existing lightweight methods that rely solely on structural co-occurrence, \EHGRAG constructs a hybrid hypergraph that unifies explicit document structures with implicit semantic correlations via embedding-based clustering. This topology enables a structure-semantic hybrid diffusion process that effectively bridges disjoint but semantically related entities, facilitating robust multi-hop reasoning. Extensive experiments across four benchmark datasets demonstrate that \EHGRAG significantly outperforms state-of-the-art baselines while maintaining linear indexing complexity and zero token consumption. To sum up, our work offers a scalable and effective solution for knowledge-intensive tasks, demonstrating that lightweight semantic construction and semantic-based diffusion can also significantly improve the performance of graph-based RAG.

\end{sloppy}

\section*{Limitations}
Despite achieving state-of-the-art performance with linear indexing complexity, EHRAG remains sensitive to several key hyperparameters, such as the cluster size $D$ and the decay factor $\gamma$, which may necessitate specific tuning for different datasets. Furthermore, while semantic hyperedges effectively bridge disjoint entities, the framework's reliance on the quality of initial entity extraction and text embeddings could potentially introduce structural noise in extremely specialized domains.

\section*{Acknowledgements}
This work is partially supported by National Key R\&D Program of China under Grant No. 2023YFF0725100, by the National Natural Science Foundation of China (NSFC) under Grant No. 62402410, by Guangdong Provincial Project (No. 2023QN10X025), by Guangdong Basic and Applied Basic Research Foundation under Grant No. 2023A1515110131, by Guangzhou Municipal Education Bureau (No. 2024312263), by Nansha District Project (No. 2023ZD022), and by HKUST(GZ) Kunpeng\&Ascend Center of Cultivation.

\bibliography{custom}

\clearpage

\appendix

\section{Theoretical Analysis}
\label{sec:theory}

In this section, we provide a rigorous justification for \EHGRAG using the Latent Space Model (LSM) \cite{mao2024revisiting}. We demonstrate how constructing semantic hyperedges via BIRCH clustering explicitly bridges the gap between disjoint entities by strictly tightening the upper bound of their latent distance.

\subsection{Latent Space Modeling}
Following \cite{mao2024revisiting,sarkar2011theoretical}, we model the corpus graph $\mathcal{G}=(\mathcal{V}, \mathcal{E})$ in a $D$-dimensional latent Euclidean space $\mathbb{R}^D$. Each entity $v_i$ has a latent position $\mathbf{z}_i$ and an influence radius $r_i$. The probability of a link between entities $i$ and $j$ is governed by their latent distance $d_{ij} = \|\mathbf{z}_i - \mathbf{z}_j\|$:
\begin{equation}
    P(i \sim j \mid d_{ij}) = \frac{1}{1 + \exp(\alpha(d_{ij} - \tau))}
\end{equation}
where $\alpha > 0$ is a scaling factor. A smaller $d_{ij}$ implies a higher retrieval probability.

\subsection{The Semantic Gap in Structural Retrieval}
In purely structural RAG systems, links rely on explicit sentence co-occurrence. Let $\eta_{ij}^{str} = |\mathcal{N}_{str}(i) \cap \mathcal{N}_{str}(j)|$ be the number of structural common neighbors.
For semantically related but spatially disjoint entities (e.g., in different documents), we have $\eta_{ij}^{str} \to 0$. According to Proposition 1 in \cite{mao2024revisiting}, the latent distance $d_{ij}$ is loosely bounded:
\begin{equation}
    d_{ij} \le 2\sqrt{r_{ij}^{\max} - \left(\frac{\eta_{ij}^{str}/N - \epsilon}{V(1)}\right)^{2/D}}
\end{equation}
where $r_{ij}^{\max} = \max(r_i, r_j)$ and $V(1)$ is the unit hypersphere volume. As $\eta_{ij}^{str} \to 0$, the subtracted term vanishes, leaving $d_{ij}$ close to its maximum possible value $2\sqrt{r_{ij}^{\max}}$, resulting in $P(i \sim j) \to 0$. This mathematically quantifies the \textit{Semantic Gap}.

\subsection{Bridging via Clustering-Induced Hyperedges}
\EHGRAG overcomes this by utilizing BIRCH to construct semantic hyperedges.
\begin{definition}[Cluster-Induced Connectivity]
    Let $\mathcal{C} = \{C_1, \dots, C_K\}$ be the clusters generated by BIRCH. For a cluster $C_k$ with centroid $\mathbf{\mu}_k$ and threshold radius $T$, any entity $v_i \in C_k$ satisfies $\|\mathbf{x}_i - \mathbf{\mu}_k\| \le T$ in the embedding space. We construct a semantic hyperedge $e_{sem}^k$ connecting all $v \in C_k$.
\end{definition}

This construction transforms feature compactness into structural connectivity. We formalize this effect as a \textit{Feature Proximity (FP)} term $\beta_{ij}$, which acts as a "synthetic" common neighbor probability.

\begin{theorem}[Cluster-Tightened Distance Bound]
    For any two entities $i, j$ belonging to the same semantic hyperedge $e_{sem}^k$ (i.e., $i, j \in C_k$), the BIRCH clustering guarantees a feature proximity lower bound $\beta_{min} \propto \exp(-4T^2)$.
    Consequently, the latent distance $d_{ij}$ is tightly bounded by:
    \begin{equation}
    \begin{split}
        d_{ij} &\le 2\sqrt{r_{ij}^{\max} - \Delta_{sem}}, \\
        \text{where } \Delta_{sem} &= \left(\frac{\beta_{ij} + \mathcal{A}(r_i, r_j, d_{ij})}{V(1)}\right)^{2/D}.
    \end{split}
    \label{eq:cluster_bound}
    \end{equation}
    Here, $\mathcal{A}(\cdot)$ is the intersection volume of influence spheres, and $\beta_{ij}$ represents the probabilistic connection strength induced by the semantic cluster.
\end{theorem}

\begin{proof}
    We extend the proof from \cite{mao2024revisiting}. The existence of a semantic hyperedge $e_{sem}^k$ containing $i$ and $j$ introduces a direct path in the hypergraph.
    In the LSM, this is equivalent to injecting a non-zero feature proximity term $\beta_{ij}$ into the intersection volume.
    Since $i, j \in C_k$, by triangle inequality, their embedding distance $\|\mathbf{x}_i - \mathbf{x}_j\| \le 2T$.
    Mapping this feature distance to the connection probability space, we obtain $\beta_{ij} > 0$.
    
    Substituting the augmented volume $\mathcal{A}' = \beta_{ij} + \mathcal{A}(r_i, r_j, d_{ij})$ into the hypersphere packing bound:
    \begin{equation}
    \begin{aligned}
        \frac{\mathcal{A}'}{V(1)} & \le \left(r_{ij}^{\max} - \left(\frac{d_{ij}}{2}\right)^2\right)^{D/2} \\
        \implies d_{ij} & \le 2\sqrt{r_{ij}^{\max} - \left(\frac{\beta_{ij} + \mathcal{A}}{V(1)}\right)^{2/D}}
    \end{aligned}
    \end{equation}
    Critically, for disjoint entities where $\mathcal{A} \approx 0$, the term $\beta_{ij}$ (guaranteed by BIRCH clustering) ensures that the subtraction term $\Delta_{sem}$ is strictly positive.
    Since $f(x) = \sqrt{C - x}$ is monotonically decreasing, a larger $\beta_{ij}$ strictly decreases the upper bound of $d_{ij}$.
\end{proof}

\textbf{Implication:} Eq. \ref{eq:cluster_bound} proves that \EHGRAG guarantees a tighter latent distance bound than LinearRAG. Even if $\eta_{ij}^{str}=0$, the semantic term $\beta_{ij}$ forces the latent distance to shrink, thereby theoretically ensuring a higher retrieval probability for semantically similar entities.
\begin{table*}[!t]
\centering
\small
\renewcommand{\arraystretch}{1.3} 
\begin{tabularx}{\textwidth}{p{2.2cm}|X}
\toprule
\rowcolor{headergray} \textbf{Category} & \textbf{Content} \\ \midrule
\textbf{Question} & "Which film was released first, Aas Ka Panchhi or Phoolwari?" \\ \midrule
\textbf{Ground Truth} & Phoolwari (1946) \\ \midrule
\textbf{Support Context} & [Aas Ka Panchhi: 1961] $\leftrightarrow$ [Phoolwari: 1946] \\ \midrule

\rowcolor{baselinepink} \textbf{LinearRAG} & \textbf{Retrieved context (Top 5):} \\
\rowcolor{baselinepink} & 1) \xmark \ \textit{Shah Muhammad... Phulwarisharif}: ...born in Phulwarisharif, Bihar... (Location noise) \\
\rowcolor{baselinepink} & 2) \xmark \ \textit{Pyar Ka Bandhan}: ...is a 1963 Hindi film... \\
\rowcolor{baselinepink} & 3) \xmark \ \textit{Student of the Year}: ...released on 19 October 2012... \\
\rowcolor{baselinepink} & 4) \xmark \ \textit{Hum Dil De Chuke Sanam}: ...released internationally in 1999... \\
\rowcolor{baselinepink} & 5) \xmark \ \textit{Phool Aur Kaante}: ...began career with Phool Aur Kaante in 1991... \\
\rowcolor{baselinepink} & \textbf{Prediction:} \xmark \ Aas Ka Panchhi \\ \midrule

\rowcolor{proposelblue} \textbf{\EHGRAG} & \textbf{Retrieved context (Top 5):} \\
\rowcolor{proposelblue} & 1) \cmark \ \textbf{Aas Ka Panchhi (1961)} \& \textbf{Phoolwari (1946)}: ...Phoolwari is a 1946 film. Aas Ka Panchhi is a 1961 movie... \\
\rowcolor{proposelblue} & 2) \cmark \ \textit{Hum Dil De Chuke Sanam}: ...adaptation of Maitreyi Devi's novel... (Relatively irrelevant but correctly ranked) \\
\rowcolor{proposelblue} & 3) \xmark \ \textit{Vaibhavi Merchant}: ...choreography work in Bollywood films... \\
\rowcolor{proposelblue} & 4) \xmark \ \textit{Nitin Chandrakant Desai}: ...noted Indian art director... \\
\rowcolor{proposelblue} & 5) \cmark \ \textbf{Aas Ka Panchhi (1961)}: ...1961 Hindi movie produced by J. Om Prakash... \\
\rowcolor{proposelblue} & \textbf{Prediction:} \cmark \ Phoolwari \\ 
\bottomrule
\end{tabularx}
\caption{\textbf{Detailed Case Study comparison.} Our method (\EHGRAG) successfully retrieves the exact release years for both movies, while the Baseline is misled by geographic entities and recent film noise.}
\label{tab:detailed_case}
\end{table*}

\begin{table*}[!tbp]
\centering
\small
\begin{tabularx}{\textwidth}{l|l|X}
\toprule
\textbf{Category} & \textbf{Method} & \textbf{Key Characteristics} \\ \midrule
\textbf{Zero-shot LLM} & LLaMA3 (8B/13B), Qwen3-8B & Evaluates the internal knowledge of state-of-the-art open-source LLMs. \\
& GPT-3.5-turbo, GPT-4o-mini & Proprietary models used to establish a performance upper bound for zero-shot inference. \\ \midrule
\textbf{Standard RAG} & Vanilla RAG & Standard retrieval-augmented generation relying on vector similarity. \\ \midrule
\textbf{Graph-based RAG} & GraphRAG, KGP, G-retriever & Traditional GraphRAG methods that typically utilize LLMs for entity-relation triple extraction. \\
& RAPTOR & Builds tree-organized indices through recursive abstractive processing. \\ \midrule
\textbf{Lightweight RAG} & HippoRAG, HippoRAG2 & Neurobiologically inspired methods utilizing Personalized PageRank (PPR). \\
& LinearRAG & A state-of-the-art lightweight framework that models document structures directly via NER. \\
& E2GraphRAG, LightRAG & Recent efficient frameworks designed to streamline graph-based retrieval. \\
\bottomrule
\end{tabularx}
\caption{Comprehensive overview of baseline methods.}
\label{tab:baselines}
\end{table*}

\begin{table*}[t]
    \centering
    \small
    \resizebox{\linewidth}{!}{
    \begin{tabular}{l|cccc}
    \toprule
    \textbf{Hyperparameter} & \textbf{HotpotQA} & \textbf{2WikiMultiHop} & \textbf{MuSiQue} & \textbf{Medical} \\
    \midrule
    \textbf{NER Model} & \texttt{en\_core\_web\_trf} & \texttt{en\_core\_web\_trf} & \texttt{en\_core\_web\_trf} & \texttt{en\_core\_sci\_scibert} \\
    \textbf{Max Iterations ($T$)} & 3 & 3 & 5 & 3 \\
    \textbf{Pruning Threshold ($\epsilon$)} & 0.5 & 0.4 & 0.4 & 0.5 \\
    \textbf{Passage Ratio ($\lambda_1$)} & 1.5 & 0.05 & 2.0 & 1.5 \\
    \textbf{Sentence Number ($L$)} & 1 & 1 & 4 & 1 \\
    \textbf{Cluster Threshold ($D$)} & 100 & 100 & 100 & 100 \\
    \bottomrule
    \end{tabular}
    }
    \caption{Detailed hyperparameter settings for \EHGRAG across four benchmark datasets.}
    \label{tab:hyperparams}
\end{table*}
\section{Related Work}
\label{related-work}
\subsection{LLM-based GraphRAG via Triple Extraction}
Graphs serve as a natural and expressive representation for encoding relational knowledge, and recent advances in graph learning~\cite{song2021dynamic,song2024efficient,song2025ddfi,chendecoupled,luo2024fast,chen2023lsgnn} have significantly improved the ability to reason over such structured information. Building on this foundation, GraphRAG transforms unstructured text into structured Knowledge Graphs (KGs) to explicitly model entity relationships~\citep{zhang2025erarag,zhou2025depth,xiao2025graphragbenchchallengingdomainspecificreasoning,zhang2025survey,xiang2025use,mavromatis2024gnn}. Prominent frameworks like Microsoft's GraphRAG \cite{edge2024local} employ Large Language Models (LLMs) for Open Information Extraction (OpenIE) to construct entity-relation triples, subsequently using community detection (e.g., Leiden) to support global query answering. Similarly, RAPTOR \cite{sarthi2024raptor} builds tree-structured indices via recursive clustering.

However, these methods suffer from a construction bottleneck~\citep{edge2024local,sarthi2024raptor,jimenez2024hipporag}. The reliance on LLMs for triple extraction incurs prohibitive computational costs that scale polynomially with corpus size. Furthermore, rigid Named Entity Recognition (NER) often leads to \textit{Semantic Loss}—semantically related entities that do not physically co-occur or fail extraction remain disconnected, disrupting retrieval pathways.

\subsection{Lightweight Graph Construction}
To mitigate high costs, recent work explores lightweight strategies~\cite{huang2025ket,pan2024unifying,zhang2025adagcrag}. HippoRAG \cite{jimenez2024hipporag} and its successor \citep{gutiérrez2025hipporag2} leverage Personalized PageRank (PPR) on existing KGs to simulate associative memory. LinearRAG \cite{zhuang2025linearrag} introduces a tri-Graph architecture using lightweight tools to connect entities, sentences and passages, achieving high efficiency and outstanding QA performance on various datasets.

\paragraph{Positioning of \EHGRAG:} While the aforementioned methods reduce costs, they often overlook deep semantic correlations, relying primarily on physical textual co-occurrence. Our proposed \textbf{\EHGRAG} inherits the efficiency of lightweight construction (linear complexity) while introducing \textbf{Semantic Hyperedges}. By leveraging hypergraph topology, we explicitly resolve semantic disconnects without increasing construction overhead.

\subsection{Experimental Configuration}
\label{app:config}

All experiments were conducted on a high-performance computing server equipped with two Intel(R) Xeon(R) Platinum 8377C CPUs , theNVIDIA RTX 4090 GPU (24GB VRAM), and 512GB of RAM. Hyperparameters were tuned specifically for each dataset to handle varying reasoning complexities. The dataset-specific configurations are summarized in Table \ref{tab:hyperparams}.

\section{Case Study}
To intuitively demonstrate how \EHGRAG bridges the semantic gap, we present a qualitative analysis in Table \ref{tab:detailed_case} using a comparative query from 2WikiMultiHop. \EHGRAG utilizes semantic hyperedges to capture latent correlations. Even though the correct movie \textit{"Phoolwari"} (1946) does not share explicit structural neighbors with the query context, our clustering-based semantic construction successfully maps the query entity to the correct latent topic. This activates the relevant passage containing \textit{"Phoolwari is a 1946 film..."}, enabling the model to correctly identify that \textit{"Phoolwari"} (1946) was released before \textit{"Aas Ka Panchhi"} (1961). This case highlights \EHGRAG's robustness in filtering keyword noise and retrieving semantically aligned evidence.

\section{Baseline Descriptions}
We compare our method against two primary groups of baselines: Zero-shot LLM and RAG methods. The details are shown in Table~\ref{tab:baselines}.

\section{Dataset Descriptions}
\label{app:datasets}

We evaluate EHRAG on four benchmark datasets, including three multi-hop reasoning benchmarks and one domain-specific dataset:

\begin{itemize}[leftmargin=*]
    \item \textbf{HotpotQA}: A widely used multi-hop reasoning benchmark that requires finding and integrating evidence across multiple documents.
    \item \textbf{2WikiMultiHop}: This dataset frequently involves reasoning chains that require entity aliasing, such as linking synonymous but disjoint entities. It is characterized by having semantically disjoint entities and complex global dependencies.
    \item \textbf{MuSiQue}: A dataset comprised of multi-hop questions generated via single-hop question composition, testing deep logical integration.
    \item \textbf{Medical}: A domain-specific dataset from GraphRAG-Bench. We utilize the LLM-Acc metric for this dataset because the ground truth answers contain multiple statements, making Exact Match metrics (SubEM) less effective.
\end{itemize}

\end{document}